\documentclass{article}

\usepackage{graphicx}
\usepackage{neurips_2021}
\usepackage{amsmath}
\usepackage{graphicx}
\usepackage{amssymb}
\usepackage{bbm}
\usepackage{booktabs}
\usepackage{float}
\usepackage{xcolor}
\usepackage{comment}
\usepackage{amsthm}
\newtheorem{definition}{Definition}
\newtheorem{lemma}{Lemma}

\newtheorem{proposition}{Proposition}

\let\emptyset\varnothing

\makeatletter
\def\blfootnote{\xdef\@thefnmark{}\@footnotetext}
\makeatother

\usepackage[utf8]{inputenc} 
\usepackage[T1]{fontenc}    
\usepackage{hyperref}       
\usepackage{url}            
\usepackage{booktabs}       
\usepackage{amsfonts}       
\usepackage{nicefrac}       
\usepackage{microtype}      
\usepackage{xcolor}         

\title{Slope and generalization properties of neural networks}

\author{Anton Johansson\textsuperscript{1*} \And  Niklas Engsner\textsuperscript{1} \And Claes Strannegård\textsuperscript{1} \And Petter Mostad\textsuperscript{1}
}

\begin{document}

\maketitle

\begin{abstract}
  Neural networks are very successful tools in for example advanced classification. 
From a statistical point of view, fitting a neural network may be seen as a kind of regression,
where we seek a function from the input space to a space of classification probabilities that follows the "general" shape of the data, but avoids overfitting by avoiding memorization of individual data points.
In statistics, this can be done by controlling the geometric complexity of the regression function. We propose to do something similar when fitting neural networks by controlling the slope of the network.

After defining the slope and discussing some of its theoretical properties, we go on to show empirically in examples, using ReLU networks, that
the distribution of the slope of a well-trained neural network classifier is generally independent of the width of the layers in a fully connected network, and that the mean of the distribution only has a weak dependence on the model architecture in general. The slope is of similar size throughout the relevant volume, and varies smoothly. It also behaves as predicted in rescaling examples.
We discuss possible applications of the slope concept, such as 
using it as a part of the loss function or stopping criterion during network training, 
or ranking data sets in terms of their complexity.
\end{abstract}

\section{Introduction}

\blfootnote{\textsuperscript{1}Chalmers University of Technology, Gothenburg, Sweden.}
\blfootnote{\textsuperscript{*}Correspondence to: Anton Johansson <johaant@chalmers.se>.}

Consider the objective of classifying items, for example images, which may be represented as points in ${\mathbb R}^{n_0}$, into $n_c$ classes. One approach is to find a map $f\in{\cal A}$ where ${\cal A}$ is the set of all continuous maps ${\mathbb R}^{n_0}\to{\mathbb R}^{n_c}$ and interpret $\operatorname{softmax}_i(f(x))$
as the probability that $x$ is in class $i$. We aim for a map that fits observed data in terms of the corresponding loss function while simultaneously avoiding overfitting to these data.

A general way to avoid overfitting is to control the "geometric complexity" of $f$, interpreting geometric complexity in a similar way as in many corresponding methods from classical statistics. A common way to limit the complexity is to define a subset ${\cal B}\subset{\cal A}$, consisting for example of all those maps expressable with a neural network with a given architecture, and to use a particular stochastic algorithm to generate a suitable $f$ that balances loss minimization and avoidance of overfitting.

As the properties we are trying to balance, the loss and the overfitting, both depend on $f$ and nothing else, we argue that control of geometric complexity should be measured in terms of properties defined directly on $f$, and not in terms of any particular neural network representation of $f$. 
Calling such properties geometric properties, we focus in this paper on what we call the slope, essentially the 
largest speed with which $f$ moves its input at a point.  

Using the hypothesis that current algorithms for training neural network classifiers are indeed successful when they control the geometric complexity of $f$ in a meaningful way, we use these algorithms to generate successful classifiers $f_1,\dots,f_k$ for a given dataset. Under our hypothesis, these functions should then have similar geometric properties, and in particular similar slope properties. These slope properties should depend only on the problem at hand, i.e., the dataset, and not on the particular neural network architecture used, except as a consequence of the approximations involved.

In this paper, we first define and study some basic properties of the slope, before studying what happens with the slope during learning for ReLU networks. We then go on to check the hypothesis above by varying the network architecture, data sets and the distance between the points in order to see how these parameters affect the slope. Finally, we discuss applications of the slope, for example to directly target a particular slope during the training of a neural network, or to use the slope for regularization.

\subsection{Related work and our contribution}
Understanding how and why neural network methods work as well as they do is clearly a vast area of research, and one that has been attacked in a number of ways. Some examples are information theoretic approaches \cite{DBLP:conf/itw/TishbyZ15}, \cite{DBLP:journals/npl/LiL21}, \cite{DBLP:conf/isit/Huang0ZW19}, 
classical statistical learning theory approaches \cite{DBLP:conf/iclr/ZhangBHRV17}, \cite{DBLP:journals/jmlr/BartlettHLM19} and others \cite{DBLP:conf/icml/ZhouF18}. A geometrical perspective has been considered in different contexts,
 e.g, by relating the manifold structure of the data distribution to generalization properties,
 see  \cite{DBLP:journals/corr/abs-1805-10451}, \cite{DBLP:journals/corr/abs-1909-11500}, or by understanding the 
inductive bias of deep neural networks by studying how the distance to the decision boundary varies as the data
representation is changed, see \cite{DBLP:journals/corr/abs-2104-14372}.

A paper taking a somewhat similar view as ours is  
\cite{yoshida2017spectral}. In this paper 
Yoshida et al impose an upper bound on the spectral norm of the local affine transformation of ReLU networks by enforcing the spectral norm of each individual weight matrix in the network to be small. This provides 
a regularization method that is related to the slope as we define it, but only indirectly. In 
\cite{yoshida2017spectral} they then go on to show that 
their regularization method has attractive properties. 

Contributions:
\begin{itemize}
    \item We define the concept of slope to capture the geometric complexity of regression maps.

\item We provide theoretical insights into properties associated with the slope, indicating how it can be used to capture aspects of the evolution and structure of the underlying geometry.

\item Additionally, we validate our theory with empirical results for ReLU networks and show
that the distribution of the slope for well-trained models is close to invariant to the width of the hidden layers in fully connected network, and that the mean of the slope distribution only has a weak dependence on the model architecture for both fully connected and convolutional networks.
\end{itemize}

\section{Notation and definitions}

\subsection{Preliminaries}
A neural network $f: \mathbb{R}^{n_0} \rightarrow \mathbb{R}^{n_c}$ will for us consist of
\begin{itemize}
    \item a sequence of positive integers $n_0,n_1,...,n_n=n_c$, where $n_1,\dots,n_n$ denote the width of the hidden layers,
    \item for $i= 1,...,n$, an $(n_i \times n_{i-1})$-dimensional matrix $W_i$ and a vector $b_i$ of length $n_{i}$, and
    \item a continuous activation function $g:{\mathbb R}\rightarrow{\mathbb R}$ applied separately to each dimension. 
\end{itemize}
We define $f^0(x)=x$ and for $i=1,\dots,n-1$ a continuous map $f^i:{\mathbb R}^{n_0}\to{\mathbb R}^{n_i}$ by setting
$$
f^i(x) = g(W_if^{i-1}(x)+b_i)
$$
while we set $f(x) = f^n(x) = W_nf^{n-1}(x)+b_n$. To use the network for classification, we apply the $\operatorname{softmax}$ function to $f(x)$ to produce an output which can be interpreted as a probability distribution on the set of $n_c$ classes. 

We are mainly concerned with neural networks with activation functions given by the Rectified Linear Units (ReLU) \cite{DBLP:conf/icml/NairH10}, referred to as \textit{ReLU networks}.
Then 
$$
g(x) = \max(0,x). 
$$

For ReLU networks we can additionally define the concept of an \textit{activation region}, the largest open connected sets $R \subset \mathbb{R}^{n_0}$ where $f$ can be represented as an affine transformation $f(x) = W_Rx + b_R, \forall x \in R$.
These regions correspond to binary patterns indicating which neurons that are activated when passing an input through the network \cite{DBLP:conf/nips/HaninR19}. While these regions possess many interesting properties, for our purposes we will mainly use that for $x$ in an activation region $R$, the Jacobian $J_f(x) = W_R$ is constant and will thus be denoted by $J_f(R)$.

\subsection{Slope}

The central geometric property we will study in this paper is the slope. All proofs are relegated to the Appendix. 

\begin{definition}
Given a continuous function $f:{\mathbb R}^{n_s}\to{\mathbb R}^{n_c}$ and some $p$ with $1\leq p\leq \infty$, we define its slope (or p-slope) at $x\in{\mathbb R}^{n_s}$ as 
$$
\operatorname{Slope}_f(x) = \sup_{v\in B^*} \left(\lim_{t\downarrow0}\frac{||f(x + t v)-f(x)||_p}{t}\right)
$$
where $||\cdot||_p$ denotes the p-norm\footnote{$||x||_p=\left(\sum_i|x_i|^p\right)^{1/p}$}, the limit is taken over positive $t$, and 
$$
B^* = \{v\in{\mathbb R}^{n_s}:||v||_p=1\}.
$$
The slope is undefined unless the limit exists for all $v\in B^*$. 
\end{definition}




\begin{proposition}
If the Jacobian $J_f(x)$ exists at $x$, then 
$$
\operatorname{Slope}_f(x) = \max_{v\in B^*}||J_f(x)v||_p = ||J_f(x)||_p. 
$$
\label{prop:slope_1}
\end{proposition}
Here $||J_f(x)||_p$ denotes the the matrix p-norm of the Jacobian. Note that when $p=2$, this is the maximum singular value of $J_f(x)$, also called the spectral norm of $J_f(x)$. 
When $p=1$ it is the maximum over the columns of $J_f(x)$ of the sum of the absolute values of the entries in the column. When $p=\infty$ it is the maximum over the rows of $J_f(x)$ of the sum of the absolute values of the entries in the row. 

\begin{proposition}
If $f$ is represented by a neural network where the activation function $g$ is continuously differentiable, then the Jacobian is a continuous function. 
If the Jacobian is a continuous function, 
then the slope is a continuous function. 
\end{proposition}

\begin{proposition}
If the Jacobian is a continuous function 
and if $\operatorname{Slope}_f(x)\leq K$ for all $x\in{\mathbb R}^{n_0}$ then for all pairs of points $x,y\in{\mathbb R}^{n_0}$, 
\begin{equation}
||f(x)-f(y)||_p \leq K ||x-y||_p
\label{eq:unifcont}
\end{equation}
\label{prop:distance}
\end{proposition}


If the output space has only one dimension, we see from Proposition~\ref{prop:slope_1} that whenever
the gradient $\triangledown f(x)$ exists at a point $x$ we have  $\operatorname{Slope}_f(x) = || \triangledown f(x)||_p$.
If the function $f$
is a type of regression function adapting to data, we would expect the slope to vary quite a bit, from zero at local extremes to larger values in between such points. 

Consider instead the case where $f(x)$ is multidimensional and the Jacobian exists.  Then we get from Proposition~\ref{prop:slope_1} that 
$$
\operatorname{Slope}_f(x) = \max_{v\in B^*} ||J_f(x)v||_p 
= \max_{v\in B^*} ||\triangledown(f\cdot v)(x)||_p 
= \max_{v\in B^*}\operatorname{Slope}_{f\cdot v}(x).
$$
In other words, we can understand the slope as follows: Take the output of $f$, project it along some direction $v$ and take the p-norm of the gradient at $x$. Then maximize over all possible directions $v$. 

If $f$ is used together with a softmax function as a classifier, we would expect that, at all points $x$, some output coordinates are increasing while others are decreasing. In other words, there will always be directions in the output space where the slope in that direction is nonzero. Thus the slope as we define it is unlikely to be 
zero anywhere, and is not so much connected to local extremes as it is to the speed at which the output changes. 

A consequence is that it is meaningful to study the average slope $f$. More specifically, 
\begin{definition}
 We define the slope of a network as the expectation of the slope when
$x$ has the distribution of the input data.
\end{definition}

Note that the distribution of the input data is unknown. However, we can estimate the quantity above by using the training data points which are a sample from the distribution. In our results, we will see that the variation of the slope across input points $x$ is often remarkably small, making the concept defined above a useful one. 

\subsection{Slopes of ReLU networks}

Our examples are all ReLU networks. For these, the Jacobian does not exist everywhere, but the slope still exists. All points $x$ inside an activation region $R$ have the same Jacobian $J_f(R)$, so we may define 
$$
\operatorname{Slope}_f(R) = \operatorname{Slope}_f(x) = ||J_f(R)||_p. 
$$

\begin{proposition}
If $f$ is represented by a ReLU network then $\operatorname{Slope}_f(x)$ exists for all $x\in{\mathbb R}^{n_0}$ and 
$$
\operatorname{Slope}_f(x) \leq\max_{R\,:\, x\in\overline{R}}\,\,\operatorname{Slope}_f(R)
$$
where $\overline{R}$ denotes the closure of $R$. 
\end{proposition}

For ReLU networks it is easy to find the Jacobian. In fact,  
\begin{equation}
J_f(x) = W_nZ_{f^{n-1}(x)}W_{n-1}\cdots Z_{f^2(x)}W_2
Z_{f^1(x)}W_1
\label{eq:prod}
\end{equation}
where $Z_{f^i(x)}$ is a diagonal matrix having 0's and 1's along its diagonal, depending on the value of $f^i(x)$. 
If $Z_{f^i(x)}=0$ for some $i$ then $J_f(x)=0$. Let us assume below that this is not the case; we then get $||Z_{f^i(x)}||_p=1$. 
In a similar way as in \cite{yoshida2017spectral} we can take the p-norm of Equation~\ref{eq:prod} to obtain 
\begin{eqnarray*}
||J_f(x)||_p &\leq& ||W_n||_p\cdot ||Z_{f^{n-1}}(x)||_p
\cdots ||W_2||_p\cdot ||Z_{f^1(x)}||_p\cdot ||W_1||_p\\
&=& ||W_n||_p\cdots ||W_2||_p\cdot ||W_1||_p
\end{eqnarray*}

Further\footnote{This follows as the spectral norm is equal to the largest singular value of $W_i$, while the Frobenius norm is equal to the square root of the sum of the squares of the singular values of $W_i$.}, we have $||W_i||_2 \leq ||W_i||_F$ where $||W_i||_F$ denotes the Frobenius norm of $W_i$, i.e., the square root of the sum of the squares of the entries of $W_i$. 
This shows that limiting the size of the entries of the $W_i$ matrices implies limiting the $||W_i||_2$ values. In turn, we have shown above that limiting $||W_i||_p$ for any $1\leq p \leq \infty$ implies limiting 
$||J_f(x)||_p$, i.e., the slope.

However, the reverse is not the case. In fact, our conjecture is that controlling the slope $||J_f(x)||_p$ is a much more fine-tuned and precise way of controlling the geometry of $f$ than standard regularization. 

\begin{proposition}
Proposition~\ref{prop:distance} holds also when $f$ is a ReLU network. 
\label{prop:distance2}
\end{proposition}


Propositions~\ref{prop:distance} and \ref{prop:distance2} indicate how the slope directly connects classification probabilities in the output space with distances in the input space. If we somehow increase the distances in the input space with with a factor $c$, we might expect the slopes of similarly well-trained classifiers to decrease with the same factor $c$. 

In fact, we will investigate this effect in the case of image resolutions. Assume the resolution of the images in an image classification dataset is changed using some algorithm. For example, images in standard datasets with $28\times28$ resolution might be rescaled to a $56\times56$ resolution, multiplying the total number of dimensions by 4. The exact change in the Euclidean distances between images will depend on the rescaling algorithm used, but as a rough estimate we may assume that the change is the same as the change of distances between independent points with a standard normal distribution when the dimension is multiplied by 4. Using Lemma~\ref{lemma2} in the Appendix we get that Euclidean distances between such points are doubled. We will compare this with empirical observations in Section \ref{sec:decreasing_img}.

\section{Slopes and learning}

Let us start with some theory: 
\begin{proposition}
Assume we have a ReLU network $f(x)=(f_1(x),\dots,f_n(x))$ followed by a softmax classifier into $n$ categories.  If the network classifies a data point $x$ correctly, the term in the loss corresponding to $x$ will decrease if $f$ is replaced by $cf$ where $c>1$ is a constant.
\end{proposition}

\begin{proposition}
For any ReLU network there exists at least one vector $v$ in the parameter space such that the gradient in the direction of $v$ corresponds to multiplying the network map $f$ with a constant $c$. 
\label{prop:gradient_and_slope}
\end{proposition}

Finally, notice that for any network map $f$ and $c>0$ we have $\operatorname{Slope}_{cf}(x)=c\operatorname{Slope}_f(x)$. 

Now, assume we are learning the parameters of a particular network, and have reached a "good model". Then, generally, most points will be correctly classified. It is then reasonable to expect that changes to the parameters along a vector like those described in the previous 
proposition will on average lead to a decline in the loss. Of course, there will often be many vectors along which the loss declines. 
However, if the training is continued for an unlimited number of epochs, the directions described in the previous theorem may become dominant. Thus, the training will lead to larger and larger slopes. 


It is a well-studied issue with the types of neural networks we are studying here that training tends to lead to larger and larger parameters. Two important methods to control this effect are regularization and batch normalization. Regularization may attempt to limit the growth of the 
values in the $W_i$ and $b_i$ parameters. Batch normalization re-centers and re-scales data values between layers. We saw in the previous section how regularization also controls the 2-slope. A similar argument can be made regarding batch normalization. 
However, we conjecture that controlling the geometric complexity of $f$ using measures defined 
in terms of $f$ (such as slope) should yield better and more precise results than using measures 
that depend on the particular neural network representation of $f$, such as standard regularization and batch normalization.

\subsection{Finding a well trained model}
\label{sec:optimal_model}
Naïve attempts to build a classifier $f$ may try to predict all points in the training data as well as possible, i.e., one may focus only on minimizing the loss as much as possible. This will lead to functions $f$ whose complexity tend to increase without bound as the amount of data increases. As discussed in the previous section 
this is connected to an ever-increasing slope for $f$. It is also a recipe for overfitting. 

In this paper we take the Bayesian viewpoint that the information content in the training data is not big enough to build a model that classifies 
perfectly on the training set and optimally on validation and test sets. Instead, one should aim for functions $f$ that weigh loss minimization against $f$ being 
"reasonable", in some sense, as a classifier. Successful classification algorithms avoid overfitting in a multitude of ways: By restricting the set of allowed 
functions $f$, by regularization that prioritizes "reasonable" $f$, by using network architectures that lead to gradients pointing toward "reasonable" $f$, 
by various stochastic mechanisms introducing noise, and by stopping the training process based on carefully chosen criteria. 

In this paper, we select, for each of a set of datasets and models that have been seen to produce fitted models with good classification accuracy
on test sets. For each dataset, we train these models a multitude of times, producing a sequence of classification functions $f_1,\dots,f_k$. 
Following the language of the paragraph above, these functions will have been produced by limiting the "unreasonableness" of the classification function in slightly different ways. 
However, our hypothesis is that these ways are sufficiently similar, and related to the specific geometric property we are studying, that we can also detect that the functions $f_1,\dots,f_k$ have similar slope properties.

This empirical investigation is performed by investigating the slope (measured with $||\cdot||_2$ for convenience) of networks trained on MNIST, KMNIST and FashionMNIST. Additionally, in order to include a non-image data set, we investigate the slope of networks trained on the Forest Cover data set. Due to time computational constraints, we do not work with the full Forest Cover data set but instead work with a random subset of 10000 data points, which are further split into 8000 training and 2000 validation points.

For each of these models we run Stochastic Gradient Descent (SGD) with a momentum of 0.8, batch size of 64 and learning rate of 0.001 for 150 epochs and the optimal model during training is chosen as the one obtained at the epoch where the validation loss was the lowest. This setup and hyperparameters are chosen so that the training proceeds long enough to give an accurate picture of the evolution of the slope, while ensuring that all models can be trained to yield accurate classifications. Unless it is otherwise mentioned, in all subsequent experiments we summarize the slope of $f$ into a single number by computing the average slope over 750 training data points chosen at random. All error-bars are obtained as the standard deviation of the slope over 5 separate runs.

\section{Results}

\subsection{Consistently increasing slopes}

A first observation is that the slope is monotonously increasing during the larger extent of the training period. This can be seen in Figure \ref{fig:increasing_slope}
where the evolution of the slope during training is shown for a variety of fully connected and convolutional networks. Each curve represents the evolution of the slope for one unique model (the exact info of the considered models can be found in Appendix \ref{sec:exp_details}). The convolutional networks are trained on MNIST, FashionMNIST, KMNIST while the fully connected network is additionally trained on the Forest Cover data set.

This continuous increase of the geometric complexity indicates that similar gradient directions to that of Proposition \ref{prop:gradient_and_slope} control the majority of the training evolution. 
It can also be seen that convolutional networks generally seem to reach higher slopes earlier than fully connected networks. 
An explanation may be that as the convolutional networks are adapted to the image analysis problem at hand, the training goes faster, i.e., takes fewer epochs, than for fully connected networks. 

The behaviour of an increasing slope is also visible in a different form in Figure \ref{fig:increasing_slope_hist} where the slope distribution at initialization and for the optimal models obtained for the Forest Cover data set can be seen. From these results it is clear that the slope of a well performing model is generally higher than that at initialization. Here it can also be seen that the distribution of the slope for the optimal model is roughly invariant to the width of the layers in the network, seen by the alignment of the estimated distributions as the layer width changes. While the distribution of the slope changes as more hidden layers are added, the mean of the distributions is relatively invariant to the model architecture.


\begin{figure}[h]
        \centering
        \includegraphics[width=\linewidth]{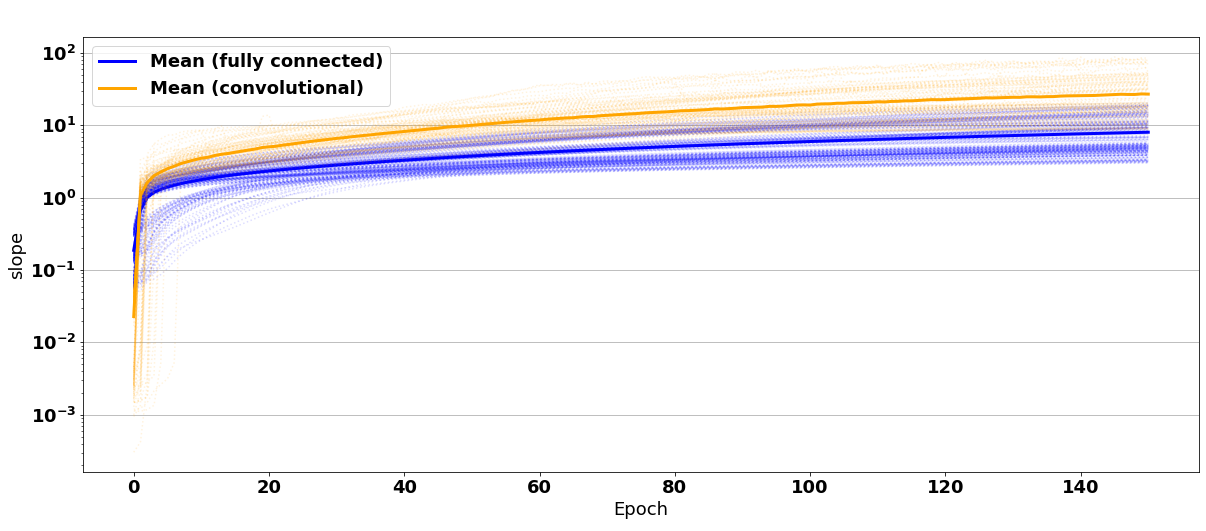}
        \caption{Increasing slope during training. Each curve represents the evolution of the slope when training one unique network for 150 epochs on either MNIST, FashionMNIST, KMNIST or the Forest Cover data set.
        The orange lines detail the evolution for convolutional networks while the blue detail it for fully connected networks. The convolutional networks are only trained on MNIST, FashionMNIST and KMNIST.}
        \label{fig:increasing_slope}
\end{figure}

\begin{figure}[h]
        \centering
        \includegraphics[width=\linewidth]{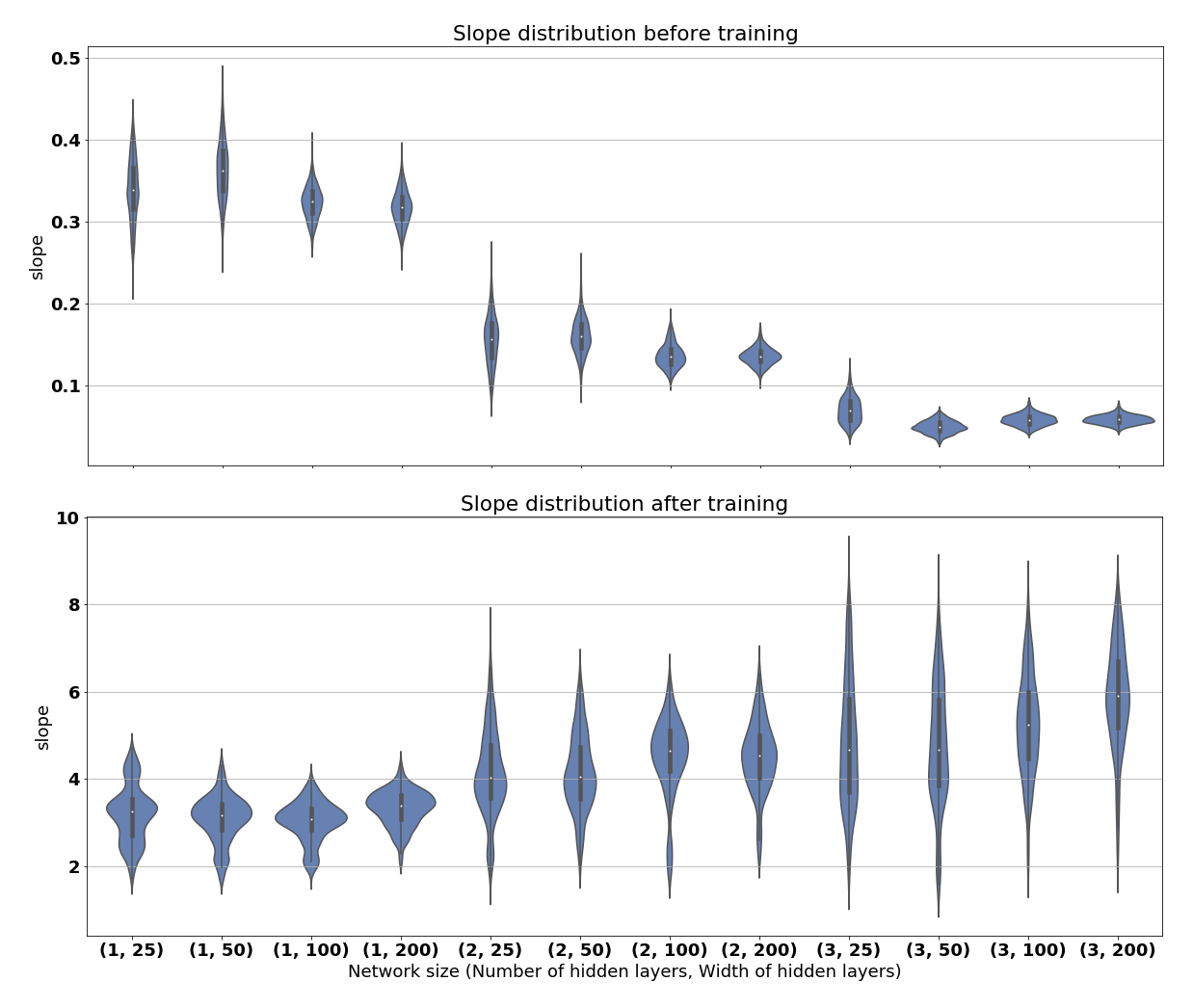}
        \caption{Violinplot of how the slope changes from initialization to the optimal model for the Forest Cover data set. Each estimated distribution is created by measuring the slope on 750 random training points. Similar plots are obtained for the other considered data sets and convolutional networks.}
        \label{fig:increasing_slope_hist}
\end{figure}

\subsection{Slopes are consistent across different SGD simulations and different network architectures}

In order to ensure that the slope contains information of the underlying geometry and that different runs of SGD produce functions $f$ with consistent slopes, we perform several repeated runs with different random seeds and measure the average slope and variance for the optimal model. 

The effect on the distributions for the optimal models when using different random seeds can be seen in Figure \ref{fig:slope_SGD}. It can be seen that the random seed has a minor effect on the shape of the estimated distribution for KMNIST, while for MNIST there are some minor discrepancies in the alignment of the distributions. This minor discrepancy can be expected given the stochastic nature of the SGD algorithm and that piecewise linear functions can locally change slope quickly without having a major effect on the overall behaviour of the function. While not shown, the effect of the random seed on the slope distribution for the optimal models trained on FashionMNIST and Forest Cover are similar to that of the effect on KMNIST.

\begin{figure}[h]
        \centering
        \includegraphics[width=\linewidth]{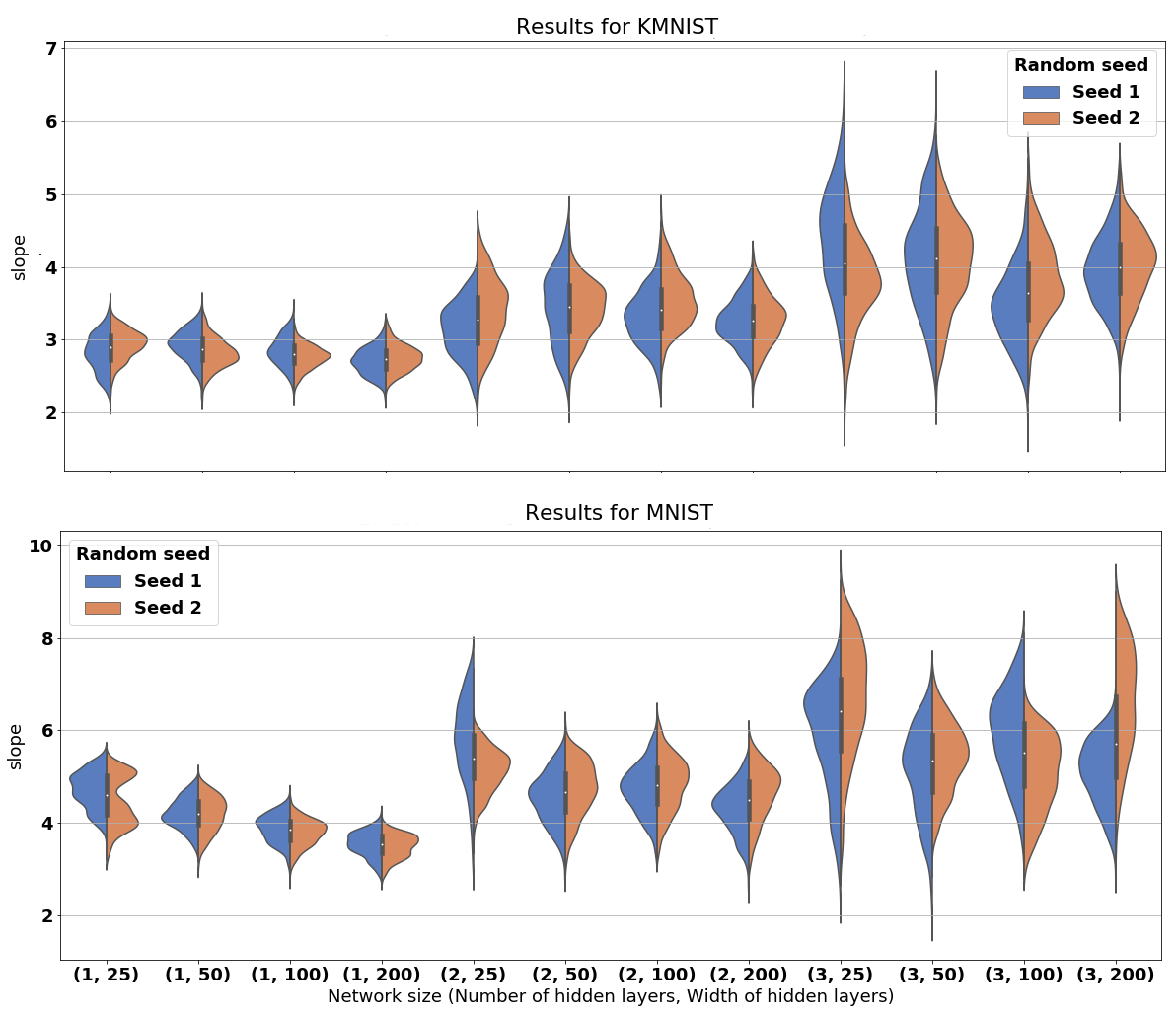}
        \caption{The estimated slope distributions for KMNIST (top) and MNIST (bottom) for two different random seeds are shown to measure the stability of the distribution of the slope to variations of the seed used for SGD.}
        \label{fig:slope_SGD}
\end{figure}

The results when only considering the effect of the seed on the mean of the distributions can be seen in Figure \ref{fig_slope_across_models_and_data} where the network structure is varied and the mean slope for the optimal model is recorded. These results show the stability of the slope in spite of random fluctuations in the learning algorithm, but they also indicate that for some data sets there might be a small range of slopes where the model will perform well, and that this range is almost independent of the network structure. 

There are however intriguing differences in slopes between fully connected and convolutional networks. The slope for the FashionMNIST data set seems to be of larger magnitude for the convolutional models while for MNIST and opposite effect can be observed.

Generally, as we observed in Figure~\ref{fig:increasing_slope}, training of convolutional networks uses fewer iterations to reach functions $f$ with higher slopes. Depending on the specifics of the dataset, this may mean that the "well fitted model" as defined in our computation is reached at an $f$ with a higher or lower slope compared to the fully connected case. Further investigation of this effect is needed. 
\begin{figure}[h]
        \centering
        \includegraphics[width=\linewidth]{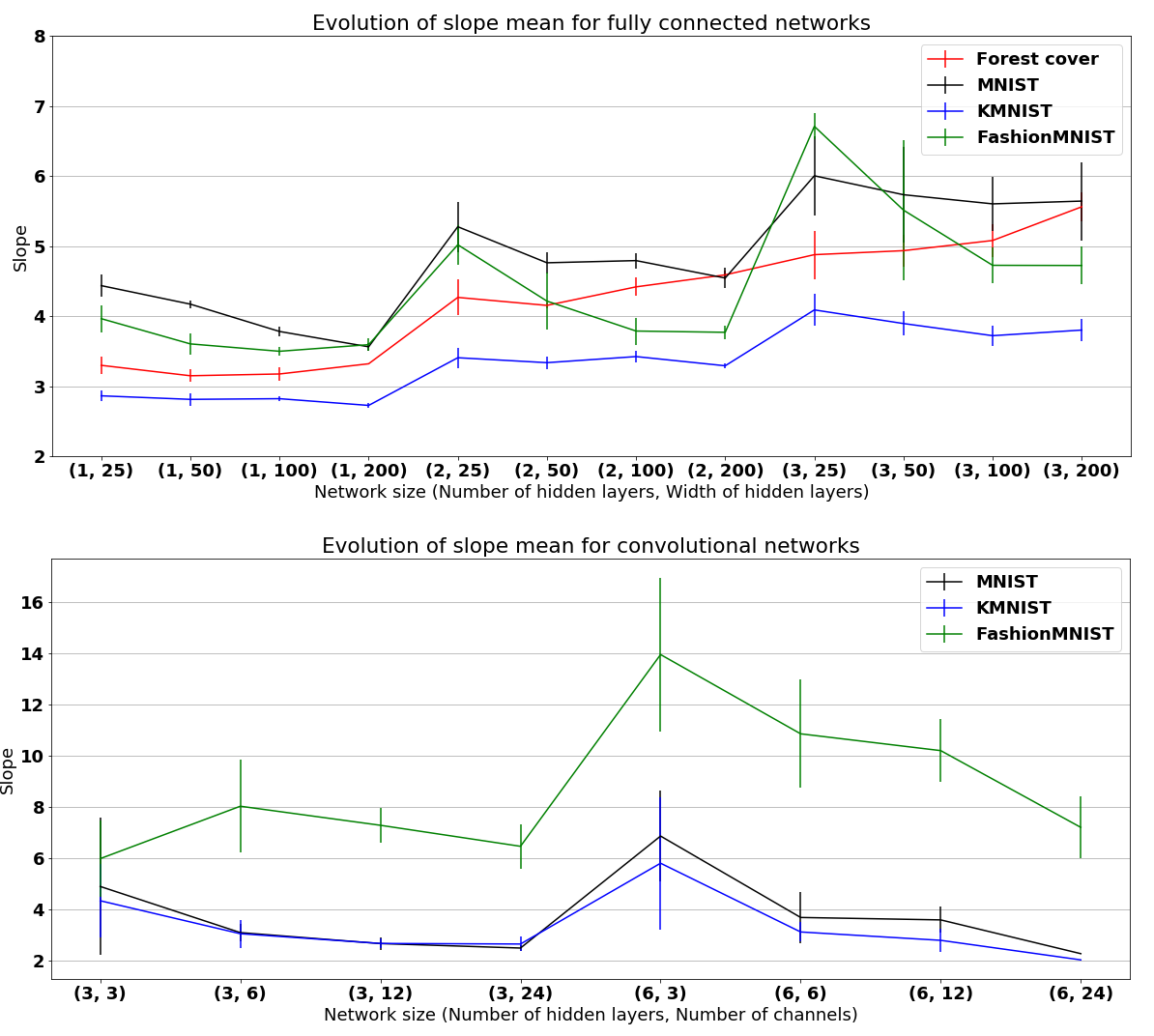}
        \caption{Experiment to measure the stability of the slope mean to variations in the structure for fully connected (top) and convolutional networks (bottom). Each curve represents how the slope mean varies as the model architecture is changed.}
        \label{fig_slope_across_models_and_data}
\end{figure}

\subsection{Slopes decrease with increasing image resolution}
\label{sec:decreasing_img}
In order to understand the relation between the slope and distance between input points,
we devise an experiment where we use bi-linear interpolation to increase the resolution of images in KMNIST, MNIST and FashionMNIST 
and investigate how the slope varies for the optimal classifiers. This setup moves input points further away from each other while it can be simultaneously argued that complexity of the classification task is preserved.
The results of the experiment when increasing the image resolution from 28x28 to 84x84 can be seen below in Figure \ref{fig:slope_image_size}. 

While the decrease is relatively linear for all three data sets, it can be seen that the decrease in slope deviates from the ideal hypothesized factor of 2 from the argument following Proposition \ref{prop:distance2}. This deviation is likely to stem from that the assumed normality required for Lemma~\ref{lemma2} does not fully capture how the true distances between input points vary, but instead only provides a rough approximation.
\begin{figure}[h]
        \centering
        \includegraphics[width=\linewidth]{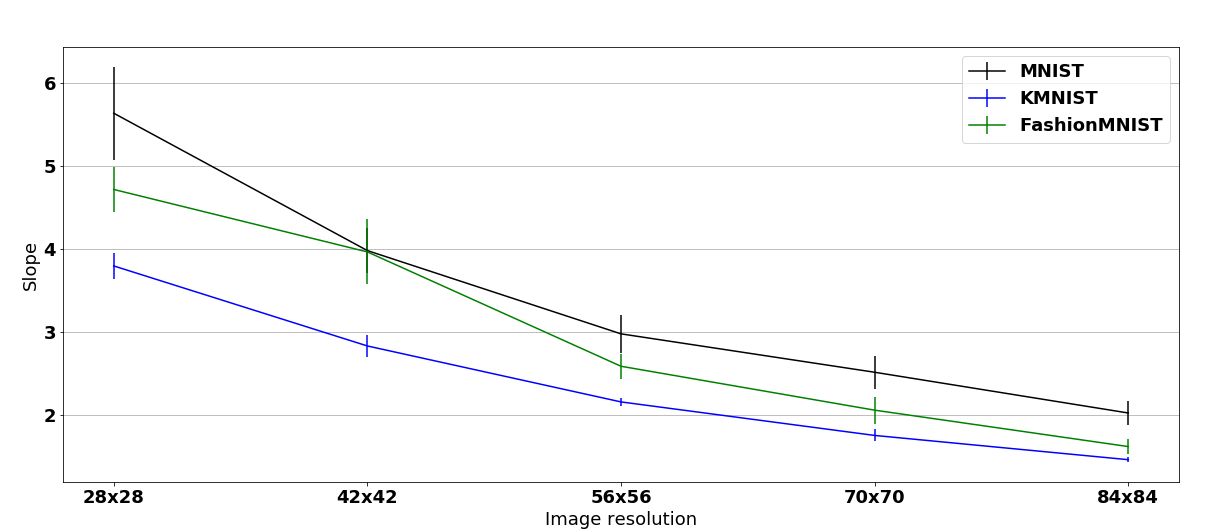}
        \caption{Relation between the slope and the distance between input points. The image resolution is increased through bi-linear interpolation.}
        \label{fig:slope_image_size}
\end{figure}

\subsection{Local variation in slope}

While the above experiments mainly consider the global properties associated with the slope, it can also be of interest to see how the slope varies locally. This is explored below in Figure \ref{fig:local_slope} where we choose 250 randomly chosen training points in FashionMNIST and sample points on concentric spheres with increasing radii and measure the relative difference between the slope for the sampled points on the spheres and the slope of the training point at the center of the sphere. The variation of the slope on each radii is summarized by sampling 500 points on each sphere and computing the relative difference to the slope at the center of the sphere. For ease of displaying the results, this is only performed for a fully connected network with 3 hidden layers, each of width 200 and the procedure is performed for the optimal network parameters, but similar results do hold for other network architectures. In the figure it can be seen that the relative slope difference is small and increasing for all radii and training points. The discontinuous nature of the slope for ReLU networks is not immediately visible but instead the slope exhibits a smooth and almost continuous change as the distance is increased.

\begin{figure}[h]
        \centering
        \includegraphics[width=\linewidth]{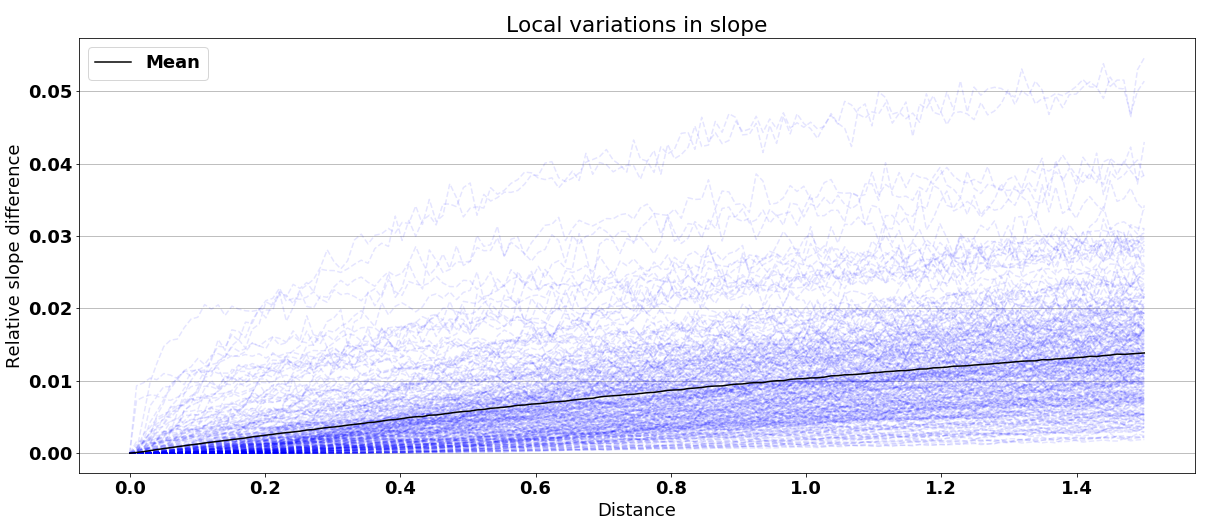}
        \caption{Experiment to measure the stability of the slope to local perturbations. Each curve is created by choosing a random training point on FashionMNIST and sampling points on concentric spheres with increasing radii around that point. The mean of the absolute slope difference between the sampled points and the training point is recorded for each radii. On each circle 250 points are sampled to estimate the difference in slope. Similar results are obtained for the other data sets.}
        \label{fig:local_slope}
\end{figure}

\section{How can the slope be used?}

We have established that well fitted classification neural networks (i.e. having a good accuracy and
limited overfitting) for a given dataset have in common similar slope properties. Thus these can be
associated with the geometry of a good classifier for the problem. A number of uses of this can be imagined.

One possibility is to view slope as a property that should be limited. In other words, any increase in the slope should be weighed against a decrease in the loss. From a Bayesian point of view one may imagine using a prior for $f$ computed from its slope. Equivalently, one might use a regularization term in the optimization computed from the slope. We aim to explore this possibility in a later paper. The explanation that good classifiers for a particular problem have similar slopes would be that this slope represents a good compromise between loss minimization and generalizability. 

Another possibility is to directly target a particular slope interval for a particular classification problem. One might first use large neural networks to establish a reasonable slope for a good classifier for a problem. Then one might use this target slope when optimizing other types of networks, for example smaller networks, or networks with particular properties such as robustness. 

A third possibility is that the network slope of a fitted network can be used as a measure of the "complexity" of a classification problem. Note that this measure would generally be different from measures based on the accuracy obtained using a particular neural network.     
 
\section{Discussion}

We have established the concept of slope for a neural network, together with some attractive theoretical properties. In examples, we have shown that the slope often does not vary much throughout the input space, that it varies smoothly, and that it tends to increase during training. In examples, we have shown that the distribution of the slope for well trained fully connected networks is almost invariant to the width of the hidden layers. Some characteristics of the distribution seems to be dependant on the number of hidden layers, but the mean of the distribution is fairly constant to these changes. There additionally seems to be some dependence on whether a fully connected or convolutional architecture is used, and further study is needed to determine the cause of this. 

Our examples have been using quite simple classification problems, such as MNIST and FashionMNIST, together with quite small neural networks. Further study is needed to determine how general our conclusions are when increasing the size of the problem and the networks. 

Additionally, while the theory holds for general values of $p$, all empirical results only consider the slope for $p=2$. While the results are conjectured to generalize to other values of $p$, this should be verified in future studies.

Nonetheless, the general idea that the generalization properties of a neural network function $f$ is determined by its geometrical properties seems supported, and the slope seems to be an example of such a geometric property. This opens up a number of interesting usages, both theoretical in connection with understanding neural network generalization properties, and practical, in terms of controlling the slopes of networks.

\bibliographystyle{plain}
\bibliography{bibliography.bib}
\clearpage

\appendix
\section{Appendix}
\subsection{Proofs}

\begin{proof} (Proof of proposition 1) 
Assuming that the Jacobian $J_f(x)$ exists we have
$$\lim_{t\downarrow0}\frac{||f(x+tv)-f(x)||_p}{t} = \lim_{t\downarrow0}||J_f(x)v + o(1)||_p = ||J_f(x)v||_p$$
where $||v|| = 1$. The slope is thus reduced to
$$
\operatorname{Slope}_f(x) = \sup_{v \in B^*} ||J_f(x)v||_p
$$ which by definition is the p-norm of the matrix $J_f(x)$. \cite{10.5555/541313}. 
\end{proof}

\begin{proof} (Proof of Proposition 2) 
The Jacobian can be constructed using the chain rule and differentiation at every layer. The differentiations are either of linear functions or of the activation functions, thus yielding continuous results. The first part follows. The second part follows from the continuity of the matrix p-norm.

\end{proof}

\begin{proof} (Proof of proposition 3) 
Assume the Jacobian is a continuous function everywhere and let $x,y\in{\mathbb R}^{n_0}$. Defining
$g(t) = f(x+t(y-x))-f(x)$ for $t\in[0,1]$ we get 
$$
f(y)-f(x) = \int_0^1g'(t)\,dt = \left(\int_0^1J_f(x+t(y-x))\,dt\right)(y-x). 
$$
Note that the integrals are taken component-wise and that the right-hand side of the equation is a matrix multiplied with a vector. Applying the p-norm to this equation and using the definition of the matrix p-norm we get 
\begin{equation}
    ||f(y)-f(x)||_p\leq\left|\left|\int_0^1J_f(x+t(y-x))\,dt\right|\right|_p||y-x||_p. 
    \label{eqproof1}
\end{equation}
The p-norm of a matrix is a convex function, so by Jensen's inequality we get 
\begin{equation}
    \left|\left|\int_0^1J_f(x+t(y-x))\,dt\right|\right|_p\leq\int_0^1||J_f(x+t(y-x))||_p\,dt.
    \label{eqproof2}
\end{equation}
If the Jacobian exists everywhere then $\operatorname{Slope}_f(x+t(y-x))\leq K$ implies
$||J_f(x+f(y-x))||_p\leq K$. Putting this together with Equations~\ref{eqproof1} and \ref{eqproof2} we get the desired result. 

\end{proof}

To prove Proposition 4, we first prove the following Lemma: 

\begin{lemma}
Let $f$ be a ReLU network. For any $x\in{\mathbb R}^{n_0}$ and $v\in B^*$ there exists a region $R_{x,v}$ and an $\epsilon_{x,v}>0$ such that 
$x+vt\in \overline{R_{x,v}}$ for all $t\in[0,\epsilon_{x,v}]$. 
\label{lemma1}
\end{lemma}

\begin{proof}
Note that the regions $R$ are associated with each component of $f^i$ in every layer being non-zero, i.e, $f_j^i(x) \neq 0, \forall (i,j)$. If for any $(i,j)$ we have $f_j^i(x) = 0$ then $x$ does not lie in a region, but instead on the "edge" of a region. For a given $x$, we thus get a separation of the neurons into three sets, the active neurons $\mathcal{A} = \{(i,j) \in \mathbbm{N} \times \mathbbm{N}: f_j^i(x) > 0\}$, the inactive neurons $\mathcal{I} = \{(i,j) \in \mathbbm{N} \times \mathbbm{N}: f_j^i(x) < 0\}$ and the edge neurons $\mathcal{E} = \{(i,j) \in \mathbbm{N} \times \mathbbm{N}: f_j^i(x) = 0\}$.

If $x$ lies inside a region then the proof is straightforward, so assume that passing $x$ through the network gives rise to the sets $\mathcal{A}, \mathcal{I}, \mathcal{E}$ with $\mathcal{E} \neq \emptyset$.
First, we will show that $x$ lies in the closure of the regions given by distributing the elements of $\mathcal{E}$ to either $\mathcal{A}$ or $\mathcal{I}$, i.e, the closure of the regions associated with the active set
 $\mathcal{A}_{\mathcal{E}}$ and the inactive set $\mathcal{I}_{\mathcal{E}}$ where we have that $e \in \mathcal{A}_{\mathcal{E}} \bigcup \mathcal{I}_{\mathcal{E}}, \forall e \in \mathcal{E}$ and $\mathcal{A} \subset \mathcal{A}_{\mathcal{E}}, \mathcal{I} \subset \mathcal{I}_{\mathcal{E}}$.
Due to the piecewise linearity of $f$, we have that locally the set $\{x \in \mathbb{R}^{n_0} : f^i_j(x) = 0, (i,j) \in \mathcal{E}\}$ is given by an intersection of hyperplanes. Each edge neuron has an associated hyperplane and each hyperplane has an associated normal vector $n_i$, which when we move in that direction will switch that neuron from an edge neuron to an active or inactive neuron. There thus exists constants $c_i\neq 0, i=1,..,|\mathcal{E}|$, a region $R$ and $\epsilon > 0$ such that the point $x + \sum_{i=1}^{|\mathcal{E}|} t c_iv_i \in R$ for all $t\in (0,\epsilon]$. Consequently, the point $x$ lies in $\overline{R}$. 

If the vector $v$ is not given as a linear combination of the normal vectors to the planes but instead points in a general direction,
then for $\epsilon>0$ we can say that the vector $x + \epsilon v$ yields a new set of active $\mathcal{A}_{x,v,\epsilon}$, inactive  $\mathcal{I}_{x,v,\epsilon}$ and edge neurons $\mathcal{E}_{x,v,\epsilon}$. Due to the continuity of ReLU networks, for small $\epsilon>0$ we have that $\mathcal{A} \subset \mathcal{A}_{x,v,\epsilon}$, $\mathcal{I} \subset \mathcal{I}_{x,v,\epsilon}$ and consequently $\mathcal{E}_{x,v,\epsilon} \subset \mathcal{E}$. Thus only some edge neurons switched to either an active or inactive state.
Since by redistributing the elements of $\mathcal{E}$ and $\mathcal{E}_{x,v,\epsilon}$ to active or inactive states we can obtain the same sets of active and inactive neurons, we get that from the argument above that there is a region $R$ such that the point $x$ and $x + tv$ lie in $\overline{R}$ for all $t\in(0,\epsilon]$, from which the desired statement follows.
\end{proof}

\begin{proof} (Proof of Proposition 4)
If $x$ and $y$ are both in a region $R$, then 
$$
f(y)-f(x) = J_f(R)(y-x). 
$$
By continuity of $f$ and convexity of $R$ this is also true when $x$ and $y$ are in the closure $\overline{R}$. Using this 
together with Lemma~\ref{lemma1} we get 
\begin{eqnarray*}
\operatorname{Slope}_f(x) &=& \sup_{v\in B^*}\left(\lim_{x\downarrow0}
\frac{||f(x+vt)-f(x)||_p}t\right) \\
&=& \sup_{v\in B^*}\left(\lim_{t\downarrow0}
\frac{||J(R_{x,v})tv||_p}{t}\right) \\
&=& \sup_{v\in B^*} ||J(R_{x,v})v||_p \\
&\leq & \max_{R: x\in\overline{R}}\,\,\sup_{v\in B^*}||J(R)v||_p \\
&=&
\max_{R: x\in\overline{R}}\,\,\operatorname{Slope}_f(R)
\end{eqnarray*}
\end{proof}

\begin{proof} (Proof of Proposition 5) 
If $f$ is piecewise linear then we construct the function $h: [0,1] \rightarrow \mathbb{R}^{n_c}$ as $h(t) = f(x(1-t) + ty)$. Assume that $h$ is constructed by $m$ different linear sections. Then there exists $m$ intervals $I_j = [t_j, t_{j+1}]$ with $\bigcup_{j=1}^m I_j = [0,1]$ such that $h$ restricted to interval $j$ is a linear function in a region $R_j$. The desired inequality can then be obtained by an application of the triangle inequality as follows,
\begin{align}
    ||f(x) - f(y)||_p = ||h(0) - h(1)||_p &\leq \sum_{j=1}^m||h(t_{j+1}) - h(t_j)||_p\\
    &= \sum_{j=1}^m||J_f(R_j)(y-x)(t_{j+1}-t_j)||_p\\
    &\leq \sum_{j=1}^m||J_f(R_j)||_p||y-x||_p(t_{j+1}-t_j) \\
    &\leq \sum_{j=1}^mK(t_{j+1} - t_j)||x-y||_p\\
    &=K||x-y||_p.
\end{align}
\end{proof}

\begin{proof} (Proof of Proposition 6) 
The relevant term in the loss function is 
$$
-\log\frac{\exp(f_i(x))}{\sum_j\exp(f_j(x))}
$$
where $i$ is the class $x$ is classified into. If $f$ classifies this point correctly we have that $f_i(x)>f_j(x)$ for all $j\neq i$. Thus, for any $c>1$, $c(f_j(x)-f_i(x))< f_j(x)-f_i(x)$. Summing over $j=1,\dots,n$ we get 
$$
\sum_jc(f_j(x)-f_i(x)) < \sum_jf_j(x)-f_i(x)
$$
which is equivalent to 
$$
-\log\frac{\exp(cf_i(x))}{\sum_j\exp(cf_j(x))} < -\log\frac{\exp(f_i(x))}{\sum_j\exp(f_j(x))}.
$$
This shows that the loss at $x$ decreases, as claimed. 
\end{proof}

\begin{proof} (Proof of proposition 7) 
Choose the vector $v$ such that a step in the direction of $v$ multiplies the weights $W_i$ and bias $b_i$ at layer $i$ and the bias at all subsequent layers $j>i$ with a constant $c>0$. For $i=1,...,n$ this step will create a new sequence of functions $f^i_c$. Since no change has been made to parameters in layers $k<i$ we have that $f^i_c = f^i$. For $k\geq i$ we obtain
$$
f^k_c(x)=
\begin{cases}
\operatorname{ReLU}(cW_kf^{k-1}(x)+cb_k) = cf^k(x),~~\textrm{for $k=i$}\\
\operatorname{ReLU}(W_kcf^{k-1}(x)+cb_k)=cf^k(x),~~\textrm{for $k>i$}
\end{cases}
$$
Thus we see that $f^n_c = cf^n = cf$ as desired.
\end{proof}

\begin{lemma}
When $x,y\sim\operatorname{Normal}_n(0, I)$ then, approximately when $n$ is large, $||x-y||\sim\operatorname{Normal}(\sqrt{2n}, 1)$.
\label{lemma2}
\end{lemma}

\begin{proof}
We get $(x-y)/\sqrt{2}\sim\operatorname{Normal}(0, I)$ and 
$||x-y||^2/2\sim\chi^2_n$, so as a first-order approximation when $n$ is large, 
$$
\frac12||x-y||^2\sim\operatorname{Normal}(n, 2n)
$$
and as a further approximation when $n$ is large, 
$$
||x-y||\sim\operatorname{Normal}(\sqrt{2n}, 1). 
$$
\end{proof}

\subsection{Experimental details}
\label{sec:exp_details}
All experiments are carried out in PyTorch \cite{NEURIPS2019_9015} and the code can be obtained on github\footnote{\href{https://github.com/antonFJohansson/slope_and_generalization}{github.com/antonFJohansson/slope\_and\_generalization} }. Every network is trained for 150 epochs and 750 random training points are chosen at the beginning of the training where the slope is subsequently measured at every epoch. To connect the slope with the generalization properties of the models we need to obtain a model that has a good generalization capability, i.e, we want to obtain an "optimal" model. This is done by choosing the optimal model to be the model at the epoch where the lowest validation loss was obtained when training the network for 150 epochs. Every network was trained with Stochastic Gradient Descent with a learning rate of 0.001, momentum of 0.8, batch size of 64 and every experiment is repeated 5 times.

The considered fully connected model architectures are obtained by varying the number of hidden layers as 1,2,3 and varying the number of neurons in the hidden layers (every hidden layer has the same number of neurons) as 25,50,100,200.

For the convolutional model architectures we vary the number of hidden layers as 3,6 and the number of channels (all hidden layers has the same number of channels) as 3,6,12,24. After the convolutional layer the feature representation is flattened and fed through a fully connected layer to the final output layer. All convolutional layers uses padding such that the feature representation maintains the same shape throughout the network.

\subsection{Info regarding the optimal models}
The validation accuracies and validation losses for all of the optimal models, both convolutional and fully connected, can be found in Table \ref{tab:acc_fc} - \ref{tab:loss_conv}.
All tables contain the mean and standard deviation of 5 runs.
While some models achieve a higher accuracy/lower loss than others, it can be argued that all models generalize. 
\begin{table}[h]
    \centering
    \begin{tabular}{lllll}
\toprule
{} &                 MNIST &                KMNIST &          FashionMNIST &          Forest Cover \\
\midrule
(1, 25)  &  $0.965 \pm 1.13e-03$ &  $0.819 \pm 3.00e-03$ &  $0.869 \pm 1.72e-03$ &  $0.718 \pm 4.60e-03$ \\
(1, 50)  &  $0.975 \pm 8.15e-04$ &  $0.856 \pm 2.16e-03$ &  $0.876 \pm 1.62e-03$ &  $0.721 \pm 2.62e-03$ \\
(1, 100) &  $0.979 \pm 1.85e-04$ &  $0.878 \pm 2.19e-03$ &  $0.882 \pm 1.39e-03$ &  $0.723 \pm 2.18e-03$ \\
(1, 200) &  $0.981 \pm 7.93e-04$ &  $0.891 \pm 7.52e-04$ &  $0.887 \pm 5.95e-04$ &  $0.728 \pm 2.08e-03$ \\
(2, 25)  &  $0.965 \pm 8.84e-04$ &  $0.822 \pm 2.64e-03$ &  $0.869 \pm 2.02e-03$ &  $0.722 \pm 1.56e-03$ \\
(2, 50)  &  $0.973 \pm 5.04e-04$ &  $0.852 \pm 2.07e-03$ &  $0.876 \pm 5.12e-04$ &  $0.727 \pm 3.41e-03$ \\
(2, 100) &  $0.977 \pm 5.84e-04$ &  $0.875 \pm 2.80e-03$ &  $0.878 \pm 6.69e-04$ &  $0.730 \pm 2.12e-03$ \\
(2, 200) &  $0.980 \pm 8.93e-04$ &  $0.886 \pm 2.49e-03$ &  $0.883 \pm 1.39e-03$ &  $0.734 \pm 2.58e-03$ \\
(3, 25)  &  $0.963 \pm 1.24e-03$ &  $0.815 \pm 3.84e-03$ &  $0.864 \pm 4.36e-03$ &  $0.724 \pm 5.11e-03$ \\
(3, 50)  &  $0.971 \pm 7.88e-04$ &  $0.846 \pm 4.88e-03$ &  $0.872 \pm 1.40e-03$ &  $0.725 \pm 3.50e-03$ \\
(3, 100) &  $0.975 \pm 6.91e-04$ &  $0.863 \pm 3.25e-03$ &  $0.876 \pm 8.08e-04$ &  $0.732 \pm 4.27e-03$ \\
(3, 200) &  $0.978 \pm 7.91e-04$ &  $0.877 \pm 2.88e-03$ &  $0.879 \pm 1.67e-03$ &  $0.739 \pm 3.98e-03$ \\
\bottomrule
\end{tabular}

    \caption{Validation accuracies for the optimal fully connected models for each data set.}
    \label{tab:acc_fc}
\end{table}

\begin{table}[h]
    \centering
    
\begin{tabular}{lllll}
\toprule
{} &                  MNIST &                 KMNIST &           FashionMNIST &           Forest Cover \\
\midrule
(1, 25)  &  $0.0019 \pm 4.97e-05$ &  $0.0097 \pm 9.81e-05$ &  $0.0058 \pm 7.17e-05$ &  $0.0106 \pm 3.68e-05$ \\
(1, 50)  &  $0.0013 \pm 4.39e-05$ &  $0.0079 \pm 8.03e-05$ &  $0.0055 \pm 3.34e-05$ &  $0.0105 \pm 6.94e-05$ \\
(1, 100) &  $0.0011 \pm 3.09e-05$ &  $0.0067 \pm 7.18e-05$ &  $0.0053 \pm 2.73e-05$ &  $0.0104 \pm 4.54e-05$ \\
(1, 200) &  $0.0010 \pm 1.61e-05$ &  $0.0060 \pm 4.32e-05$ &  $0.0051 \pm 2.57e-05$ &  $0.0103 \pm 3.88e-05$ \\
(2, 25)  &  $0.0019 \pm 4.87e-05$ &  $0.0098 \pm 2.99e-04$ &  $0.0058 \pm 6.17e-05$ &  $0.0105 \pm 9.92e-05$ \\
(2, 50)  &  $0.0014 \pm 2.82e-05$ &  $0.0082 \pm 2.80e-05$ &  $0.0055 \pm 3.31e-05$ &  $0.0103 \pm 7.35e-05$ \\
(2, 100) &  $0.0012 \pm 5.71e-05$ &  $0.0070 \pm 1.39e-04$ &  $0.0054 \pm 2.72e-05$ &  $0.0101 \pm 6.72e-05$ \\
(2, 200) &  $0.0010 \pm 3.29e-05$ &  $0.0064 \pm 1.19e-04$ &  $0.0053 \pm 5.91e-05$ &  $0.0100 \pm 3.60e-05$ \\
(3, 25)  &  $0.0020 \pm 4.42e-05$ &  $0.0103 \pm 2.74e-04$ &  $0.0061 \pm 1.65e-04$ &  $0.0105 \pm 7.99e-05$ \\
(3, 50)  &  $0.0016 \pm 4.50e-05$ &  $0.0086 \pm 2.00e-04$ &  $0.0058 \pm 1.01e-04$ &  $0.0103 \pm 6.18e-05$ \\
(3, 100) &  $0.0013 \pm 2.80e-05$ &  $0.0076 \pm 2.09e-04$ &  $0.0056 \pm 4.08e-05$ &  $0.0101 \pm 6.51e-05$ \\
(3, 200) &  $0.0012 \pm 3.34e-05$ &  $0.0069 \pm 4.06e-05$ &  $0.0055 \pm 2.60e-05$ &  $0.0099 \pm 4.22e-05$ \\
\bottomrule
\end{tabular}

    \caption{Validation losses for the optimal fully connected models for each data set.}
    \label{tab:loss_fc}
\end{table}

\begin{table}[h]
    \centering
    \begin{tabular}{llll}
\toprule
{} &                 MNIST &                KMNIST &          FashionMNIST \\
\midrule
(3, 3)  &  $0.980 \pm 1.20e-03$ &  $0.851 \pm 5.88e-03$ &  $0.889 \pm 3.32e-03$ \\
(3, 6)  &  $0.983 \pm 1.24e-03$ &  $0.879 \pm 6.71e-03$ &  $0.900 \pm 2.94e-03$ \\
(3, 12) &  $0.985 \pm 1.13e-03$ &  $0.899 \pm 3.31e-04$ &  $0.903 \pm 2.13e-03$ \\
(3, 24) &  $0.986 \pm 9.83e-04$ &  $0.907 \pm 1.34e-03$ &  $0.908 \pm 6.95e-04$ \\
(6, 3)  &  $0.983 \pm 9.60e-04$ &  $0.862 \pm 8.26e-03$ &  $0.893 \pm 2.12e-03$ \\
(6, 6)  &  $0.984 \pm 1.84e-03$ &  $0.886 \pm 3.11e-03$ &  $0.898 \pm 4.92e-03$ \\
(6, 12) &  $0.985 \pm 1.56e-03$ &  $0.896 \pm 6.14e-03$ &  $0.900 \pm 2.45e-03$ \\
(6, 24) &  $0.985 \pm 8.50e-04$ &  $0.904 \pm 4.00e-04$ &  $0.900 \pm 2.46e-03$ \\
\bottomrule
\end{tabular}
    \caption{Validation accuracies for the optimal convolutional networks for each data set.}
    \label{tab:acc_conv}
\end{table}

\begin{table}[h]
    \centering
    \begin{tabular}{llll}
\toprule
{} &                  MNIST &                 KMNIST &           FashionMNIST \\
\midrule
(3, 3)  &  $0.0010 \pm 1.05e-04$ &  $0.0087 \pm 2.20e-04$ &  $0.0050 \pm 1.44e-04$ \\
(3, 6)  &  $0.0008 \pm 7.39e-05$ &  $0.0070 \pm 2.33e-04$ &  $0.0045 \pm 1.56e-04$ \\
(3, 12) &  $0.0007 \pm 3.35e-05$ &  $0.0061 \pm 1.16e-04$ &  $0.0044 \pm 1.06e-04$ \\
(3, 24) &  $0.0006 \pm 4.04e-05$ &  $0.0056 \pm 1.77e-04$ &  $0.0043 \pm 6.18e-05$ \\
(6, 3)  &  $0.0009 \pm 6.79e-05$ &  $0.0079 \pm 5.92e-04$ &  $0.0048 \pm 1.94e-04$ \\
(6, 6)  &  $0.0008 \pm 8.20e-05$ &  $0.0070 \pm 1.61e-04$ &  $0.0045 \pm 2.35e-04$ \\
(6, 12) &  $0.0008 \pm 4.61e-05$ &  $0.0064 \pm 3.00e-04$ &  $0.0045 \pm 9.90e-05$ \\
(6, 24) &  $0.0007 \pm 2.98e-05$ &  $0.0058 \pm 1.85e-04$ &  $0.0045 \pm 8.23e-05$ \\
\bottomrule
\end{tabular}
    \caption{Validation losses for the optimal convolutional models for each data set}
    \label{tab:loss_conv}
\end{table}

\end{document}